\documentclass[journal]{IEEEtran}

\usepackage{amsmath,mathtools,amsthm}
\usepackage{authblk}
\usepackage{amssymb, bm}
\usepackage[mathscr]{eucal}
\usepackage{bm}
\usepackage{bbm}
\usepackage[shortlabels]{enumitem}
\usepackage{hyperref}
\usepackage{booktabs}
\usepackage[usenames,dvipsnames]{xcolor}
\usepackage[linesnumbered,ruled]{algorithm2e}

\usepackage{tikz}
\usepackage{graphicx, caption}
\usepackage{subcaption}

\newcommand{\prob}{\mathbb{P}}
\newcommand{\Prob}[1]{\prob\left(#1\right)}

\newcommand{\expec}{\mathbb{E}}
\newcommand{\Exp}[1]{\expec\left[#1\right]}

\newcommand{\Var}[1]{\textup{Var}\left(#1\right)}

\newcommand{\wlim}{\ensuremath{\stackrel{W}{\longrightarrow}}}

\newcommand{\ind}[1]{\mathbbm{1}_{\left\{#1\right\}}}

\newcommand{\sss}[1]{\scriptscriptstyle{#1}}

\newcommand\abs[1]{\left|#1\right|}
\newcommand{\me}{\textup{e}}

\newcommand{\dd}{\textup{d}}

\newcommand{\mm}{{\sss{(-)}}}
\newcommand{\pp}{{\sss{(+)}}}

\allowdisplaybreaks

\newtheorem{theorem}{Theorem}

\newtheorem{lemma}[theorem]{Lemma}

\theoremstyle{definition}
\newtheorem{example}{Example}

\allowdisplaybreaks

\graphicspath{{Figures/}}
\title{Efficient inference in stochastic block models with vertex labels}

\begin{document}
	\author{Clara Stegehuis \thanks{C. Stegehuis is with Eindhoven University of Technology.} and Laurent Massouli\'e	\thanks{L. Massouli\'e is with Microsoft Research - INRIA joint centre.}}

	\maketitle
	\begin{abstract}
	We study the stochastic block model with two communities where vertices contain side information in the form of a vertex label. These vertex labels may have arbitrary label distributions, depending on the community memberships. 
	We analyze a linearized version of the popular belief propagation algorithm. We show that this algorithm achieves the highest accuracy possible whenever a certain function of the network parameters has a unique fixed point. Whenever this function has multiple fixed points, the belief propagation algorithm may not perform optimally. We show that increasing the information in the vertex labels may reduce the number of fixed points and hence lead to optimality of belief propagation.
	\end{abstract}
	
	\section{Introduction}
	Many real-world networks contain community structures: groups of densely connected nodes. Finding these group structures based on the connectivity matrix of the network is a problem of interest, and several algorithms have been developed to extract these community structures, see~\cite{fortunato2010} for an overview. In many applications however, the network contains more information than just the connectivity matrix. For example, the edges can be weighted, or the vertices can carry information. This extra network information may help in extracting the community structure of the network. In this paper, we study the setting where the vertices have labels, which arises in particular when vertices can be distinguished into different types. For example, in social networks vertex types may include the interests of a person, the age of a person or the city a person lives in. We investigate how the knowledge of these vertex types helps us in identifying the community structures. 
	
	We focus on the stochastic block model (SBM), a popular random graph model to analyze community detection problems~\cite{holland1983,decelle2011b,snijders1997}. In the simplest case, the stochastic block model generates a random graph with 2 communities. First, the vertex set is partitioned into two communities. Then, two vertices in communities $i$ and $j$ are connected with probability $M_{ij}$ for some connection probability matrix $M$. To include the vertex labels, we then attach a label to every vertex, where the label distribution depends on the community membership of the vertex. 
	
	In the stochastic block model with two equally sized communities, it is not always possible to infer the community structure from the connectivity matrix. A phase transition occurs at the so-called Kesten-Stigum threshold $\lambda_2^2d=1$, where $\lambda_2$ is the second largest eigenvalue of a matrix related to the connectivity matrix and $d$ the average degree in the network. Underneath the Kesten-Stigum threshold, no algorithm is able to infer the community memberships better than a random guess, even though a community structure may be present~\cite{mossel2014}. In this setting, it is even impossible to discriminate between a graph generated by the stochastic block model and an Erd\H{o}s R\'enyi random graph with the same average degree, even though a community structure is present~\cite{mossel2012}. Above the Kesten-Stigum threshold, the communities can be efficiently reconstructed~\cite{massoulie2014,mossel2017}.  
	
	A popular algorithm for community detection is belief propagation~\cite{decelle2011} (BP). This algorithm starts with initial beliefs on the community memberships, and iterates until these beliefs converge to a fixed point. Above the Kesten-Stigum threshold, a fixed point that is correlated with the true community memberships is believed to be the only stable fixed point, so that the algorithm always converges to that fixed point. Underneath the Kesten-Stigum threshold, the fixed point correlated with the true community memberships becomes unstable and the belief propagation algorithm will in general not result in a partition that is correlated with the true community memberships. However, when the belief propagation algorithm is initialized with the real community memberships, there is still a regime of the parameters where the fixed point correlated with the true community spins can be distinguished from the other fixed points. In this regime, community detection is believed to be possible (for example by exhaustive search of the state space), but not in polynomial time.
	
When the two communities are equally sized (the symmetric stochastic block model), the phase where community detection may only be possible by non-polynomial time algorithms is not present~\cite{mossel2017,massoulie2014,mossel2014}. In the case of unbalanced communities (the asymmetric stochastic block model) it has been shown that it is possible to infer the community structure better than random guessing even below the Kesten-Stigum threshold~\cite{neeman2014}. Thus, according to the conjecture of~\cite{decelle2011}, a regime where community detection is possible but not in polynomial time may be present in the case of two unbalanced communities.
	
	In this paper, we investigate the performance of the belief propagation algorithm on the asymmetric stochastic block model when vertices contain side information. We are interested in the fraction of correctly inferred community labels by the algorithm, and we say that the algorithm performs optimally if it achieves the highest possible fraction of correctly inferred community labels among all algorithms.
	Some special cases of stochastic block models with side information have already been studied. One such case is the setting where a fraction $\beta$ of the vertices reveals its true group membership~\cite{zhang2014}. Typically, it is assumed that the fraction of vertices that reveal their true membership tends to zero when the graph becomes large~\cite{zhang2014,caltagirone2017}. In this setting, a variant of the belief propagation algorithm including the vertex labels seems to perform optimally in a symmetric stochastic block model~\cite{zhang2014}, but may not perform optimally if the communities are not of equal size~\cite{caltagirone2017}.
	Another special case of the label distribution is when the observed labels are a noisy version of the community memberships, where a fraction of $\beta$ vertices receives the label corresponding to their community, and a fraction of $1-\beta$ vertices receives the label corresponding to the other community. It was conjectured in~\cite{mossel2016} that for this label distribution the belief propagation algorithm always performs optimally in the symmetric stochastic block model.
	
	\paragraph*{Our contribution}
	We focus on asymmetric stochastic block models with arbitrary label distributions, generalizing the above examples.
	\begin{itemize}
		\item 
	We provide an algorithm that uses both the label distribution and the network connectivity matrix to estimate the group memberships. 
	\item 
	The algorithm is a local algorithm, which means that it only depends on a finite neighborhood of each vertex. In particular, this implies that the running time of the algorithm is linear, allowing it to be used on large networks. The algorithm is a variant of the belief propagation algorithm, and a generalization of the algorithms provided in~\cite{mossel2016,caltagirone2017} to include arbitrary label distributions and an asymmetric stochastic block model. 
	\item 
	In a regime where the average vertex degrees are large, we obtain an expression for the probability that the community of a vertex is identified correctly. 
	Furthermore, we show that this algorithm performs optimally if a function of the network parameters has a unique fixed point. 
	\item 
	Similarly to belief propagation without labels, we show that when multiple fixed points exist, the belief propagation algorithm may not converge to the fixed point containing the most information about the community structure. This phenomenon was previously observed in a setting where the information carried by the vertex labels tends to zero in the large graph limit~\cite{caltagirone2017}, but we show that this may also happen if the information carried by the vertex labels does not tend to zero. The existence of multiple fixed points either indicates that the optimal fixed point can still be found by an exhaustive search of the partition space or it may indicate that no algorithm is able to detect the community partition. 
	\item 
	We show that increasing the correlation between the vertex covariate and the community structure changes the number of fixed points of the BP algorithm for a specific example of node covariates. In particular it is possible that the BP algorithm does not converge to the fixed point that is the most informative on the vertex spins if the correlation between the vertex covariates and the vertex spins is small, but that BP does converge to this fixed point if the vertex labels contain more information on the vertex spins. This shows that including node covariates for community detection is helpful, and that it may significantly improve the performance of polynomial time algorithms for community detection. 
	\end{itemize}

	We start by showing with an example that in some cases vertex labels allow us to efficiently detect communities even below the Kesten-Stigum threshold.

\begin{example}\label{ex:SBM}
	We now present a simple example where it is not possible to detect communities using the connectivity matrix only, but where it is possible when we also use knowledge of the vertex labels. Consider an SBM with four communities of size $n/4$, 1,2,3 and 4, where the probability that a vertex in community $i$ connects to a vertex in community $j$ is given by $M_{ij}$. Here $M$ is the connection probability matrix defined as
	\begin{equation*}
		M=\frac{1}{n}\begin{bmatrix}
			2a & 2b & a+b & a+b\\
			2b & 2a & a+b & a+b\\
			a+b & a+b & 2a & 2b\\
			a+b & a+b & 2b & 2a
		\end{bmatrix}.
	\end{equation*}
	The nonzero eigenvalues of this matrix are given by $2(a-b)/n$ (appears with multiplicity two) and $4(a+b)/n$. Community detection in this example is not able to obtain a partition that is better than a random guess below the Kesten-Stigum threshold, which is
	\begin{equation*}
		(a-b)^2<4( a+b).
	\end{equation*}
	
	Now suppose that all vertices in communities 1 and 2 have label $\ell_1$, and all vertices in communities 3 and 4 have labels $\ell_2$. Then, there are $n/2$ vertices with label $\ell_1$ and $n/2$ with label $\ell_2$. Thus, using the labels of the communities alone we cannot distinguish between vertices in community 1 and 2 or between vertices in communities 3 and 4. Thus, using the labels only, we can only correctly infer at most half of the community spins. 
	
	Now suppose we split the network into two smaller networks based on the label of the vertices. Then, we obtain two small networks with connection probability matrices 
	\begin{equation*}
	\frac{1}{n}	\begin{bmatrix}
			2a &2b\\
			2b& 2a
		\end{bmatrix}.
	\end{equation*}
	Thus, community detection can achieve a partition that is better than a random guess in these two networks as long as $(a-b)^2>2(a+b)$, i.e., above the corresponding Kesten-Stigum threshold. Thus, in the regime
	\begin{equation*}
		2(a+b)<(a-b)^2<4(a+b)
	\end{equation*}
	it is impossible to infer the community structure better than a random guess without information about the vertex labels, or when using only the vertex labels. However, when using the vertex label information combined with the underlying graph structure, one can infer the community structure of strictly more than half of the vertices correctly. 
	\end{example}
%

	\paragraph*{Notation}
	We say that a sequence of events $(\mathcal{E}_n)_{n\geq 1}$ happens with high probability (w.h.p.) if $\lim_{n\to\infty}\Prob{\mathcal{E}_n}=1$. Furthermore, we write $f(n)=o(g(n))$ if $\lim_{n\to\infty}f(n)/g(n)=0$, and $f(n)=O(g(n))$ if $|f(n)|/g(n)$ is uniformly bounded, where $(g(n))_{n\geq 1}$ is nonnegative. 
	
	\subsection{Model}
	Let $G$ be a labeled SBM with two communities. That is, every vertex $i$ has a spin $\sigma_i\in\{+,-\}$, where $\Prob{\sigma_i=+}=p$ independently for all $i$. 
	Each pair of nodes $(i,j)$ is connected with probability $da/n$ if $\sigma_i=\sigma_j=+$, with probability $dc/n$ if $\sigma_i=\sigma_j=-$, and with probability $db/n$ if $\sigma_i\neq \sigma_j$, so that $d$ controls the average degree in the graph. When the communities do not have equal degrees, partitioning vertices based on their degrees already results in a community detection algorithm that correctly identifies the spin of a vertex with probability at strictly larger than $1/2$~\cite{caltagirone2017}. We therefore assume that all vertices have the same average degree, that is
	\begin{equation}\label{eq:avgdeg}
	pa+(1-p)b = pb+(1-p)c=1,
	\end{equation}
	so that the average degree is $d$.

	Beside the vertex spins, every vertex has a label attached to it. Let $\mathcal{L}$ be a finite set of labels. Then vertices in community + have label $\ell\in\mathcal{L}$ with probability $\mu(\ell)$, and vertices in community~- have label $\ell$ with probability $\nu(\ell)$. 
	
	For an estimator of the community spins $T$, let $T_i(G)\in\{ +,-\}$ be the estimated label of vertex $i$ in graph $G$ under estimator $T$. We then define the success probability of an estimator $T$ as
	\begin{equation}\label{eq:psucc}
	\begin{aligned}[b]
	P_{succ}(T)=\frac{1}{n}\sum_{i=1}^{n}& \big(\Prob{T_i(G)=+\mid \sigma_i=+}\\
	& +\Prob{T_i(G)=-\mid\sigma_i=-}-1\big),
	\end{aligned}
	\end{equation}
	where the subtraction of -1 is to give zero performance measure to estimators that do not depend on the graph structure $G$. 
	Let $s_0$ be a uniformly chosen vertex. By~\cite[Proposition 3]{caltagirone2017},
	\begin{equation}\label{eq:dtvpsucc}
	P_{succ}(T)=d_{TV}(P_+,P_-),
	\end{equation}
	where $P_+$ and $P_-$ are the conditional distributions of $G$, given that $\sigma_{s_0}=+$ and $\sigma_{s_0}=-$ respectively and $d_{TV}$ denotes the total variation distance. 
	We say that the community detection problem is solvable if the estimator $T^{opt}$ maximizing~\eqref{eq:psucc} satisfies
	\begin{equation}\label{eq:popt}
		\liminf_{n\to\infty} P_{succ}(T^{opt})>0.
	\end{equation}
	Note that the estimator $T_1$ that estimates community one if $\mu(\ell)>\nu(\ell)$ and community two otherwise has a success probability of
	\begin{equation}\label{eq:onlylabels}
		\begin{aligned}[b]
		P_{succ}(T_1)& =\sum_{\ell:\mu(\ell)>\nu(\ell)}\mu(\ell)+\sum_{\ell:\nu(\ell)\leq \mu(\ell)}\nu(\ell)-1\\
		& =\sum_{\ell}\max(\mu(\ell),\nu(\ell))-1\\
		& =\sum_{\ell}\left(\max(\mu(\ell),\nu(\ell))-\tfrac{1}{2}(\mu(\ell)+\nu(\ell))\right)\\
		&=\tfrac{1}{2}\sum_{\ell}\abs{\mu(\ell)-\nu(\ell)}=d_{TV}(\mu,\nu)
		\end{aligned}
	\end{equation} 
	Thus, the community detection problem is always solvable when $d_{TV}(\mu,\nu)>0$. Furthermore, an estimator $T$ performs better when combining the network data and the vertex labels than when only using the vertex labels if
	\begin{equation*}
	P_{succ}(T)>d_{TV}(\mu,\nu).
	\end{equation*}
	
	\subsection{Labeled Galton-Watson trees}\label{sec:GW}
	A widely used algorithm to detect communities in the stochastic block model is Belief Propagation~\cite{decelle2011}. The algorithm computes the belief that a specific vertex $i$ belongs to community $+$, given the beliefs of the other vertices. Because the stochastic block model is locally tree-like, we study a Galton-Watson tree that behaves similar to the labeled stochastic block model. 
	We denote this labeled Galton-Watson tree by $(\mathcal{T},s_0,\sigma,L)$, where $\mathcal{T}$ is a Galton-Watson tree rooted at $s_0$ with a Poisson$(d)$ offspring distribution. Each vertex $i$ in the tree has two covariates, $\sigma_i\in\{+,-\}$ and $L_i\in\mathcal{L}$. Here $\sigma_i$ denotes the spin of the node, and $L_i$ denotes the vertex label of node $i$. The root $s_0$ has spin $\sigma_{s_0}=+$ with probability $p$ and spin - with probability $1-p$. Given the label $\sigma_P$ of the parent of node $i$, the probability that $\sigma_i=\sigma_P$ is $pa/d$ if $\sigma_i=+$ and $(1-p)c/d$ if $\sigma_i=-$.
	Given $\sigma_i=+$, $L_i=\ell$ with probability $\mu(\ell)$, whereas given $\sigma_i=-$, $L_i=\ell$ with probability $\nu(\ell)$. 
	
	Let $\mathcal{T}_r^{(s_0,L)}$ denote such a tree of depth $r$ rooted at $s_0$, where the labels $L_i$ are observed, but the spins of the nodes are not observed. Let $\partial_i$ denote the set of children of vertex $i$. Then Bayes' rule together with~\eqref{eq:avgdeg} yields
	\begin{equation*}
\begin{aligned}[b]
	& \frac{\Prob{\sigma_{s_0}=+\mid\mathcal{T}_r^{(s_0,L)}}}{\Prob{\sigma_{s_0}=-\mid\mathcal{T}_r^{(s_0,L)}}}
	=\frac{\mu(L_{s_0})p}{\nu(L_{s_0})(1-p)}\\
	& \times \frac{\prod_{j\in\partial_{s_0}}a\Prob{\sigma_j=+\mid \mathcal{T}_{r-1}^{(j,L)}}+b\Prob{\sigma_j=-\mid \mathcal{T}_{r-1}^{(j,L)}}}{\prod_{j\in\partial_{s_0}}b\Prob{\sigma_j=+\mid \mathcal{T}_{r-1}^{(j,L)}}+c\Prob{\sigma_j=-\mid \mathcal{T}_{r-1}^{(j,L)}}}.
\end{aligned}
	\end{equation*}
	 If we define
	\begin{equation}
	\xi_r^{(s_0)}=\log\left(\frac{\Prob{\sigma_{s_0}=+\mid\mathcal{T}_r^{(s_0,L)}}}{\Prob{\sigma_{s_0}=-\mid\mathcal{T}_r^{(s_0,L)}}}\right),
	\end{equation}
	we can write the recursion
	\begin{equation}\label{eq:xirecurs}
	\xi_r^{(s_0)}=h(L_{s_0})+w+\sum_{j\in\partial_{s_0}}f(\xi_{r-1}^{(j)}),
	\end{equation}
	with $w=\log(p/(1-p))$,
	\begin{equation}
	f(x)=\log\left(\frac{a\me^x+b}{b\me^x+c}\right)
	\end{equation}
	and 
	\begin{equation}
	h(\ell )=\log(\mu(\ell)/\nu(\ell)).
	\end{equation}
	
	\subsection{Local algorithms}
	 A local algorithm is an algorithm that bases the estimate of the spin of a vertex $i$ only on the neighborhood of vertex $i$ of radius $t$. In general, local algorithms are not able to obtain a success probability~\eqref{eq:psucc} larger than zero in the stochastic block model~\cite{kanade2016}, so that an estimator based on a local algorithm does not satisfy~\eqref{eq:popt}. However, when a vanishing fraction of vertices reveals their labels, a local algorithm is able to achieve the maximum possible success probability~\eqref{eq:psucc} when the parameters of the stochastic block model are above the Kesten Stigum threshold~\cite{kanade2016}. 
	
	\subsection{Local, linearized belief Propagation}
	The specific local algorithm we consider is a version of the widely used belief propagation~\cite{decelle2011}. Algorithm~\ref{alg:BPlocal} uses the observed labels to initialize the belief propagation, and then updates the beliefs as in~\eqref{eq:xirecurs}. Here $\mathcal{N}_i$ denotes the neighbors of vertex $i$. Since the algorithm only uses the neighborhood of vertex $i$ up to depth $t$, it is indeed a local algorithm. Note that Algorithm~\ref{alg:BPlocal} does require knowledge of all parameters of the stochastic block model: $p,a,b,c,d$ as well as the label distributions $\mu$ and $\nu$. Furthermore, if the underlying graph $G$ is a tree, then~\eqref{eq:ri} is the same as~\eqref{eq:xirecurs}. 
	
	\begin{algorithm}[tb]
		Set $R_{i\to j}^0=\log(\mu(L_i)/\nu(L_i))$ for all $i\in[n]$ and $j\in\mathcal{N}_i$.\\
		\For{$k=1,\dots,t-1$}{
			For all $(i,j)\in E$ let 
	\begin{equation}
	R_{i\to j}^k=h(L_i)+w+\sum_{v\in\mathcal{N}_i\setminus\{j\}}f(R_{v\to i}^{k-1}).
	\end{equation}	
	}
For all $i\in[n]$ set
\begin{equation}\label{eq:ri}
R_i^t=h(L_i)+w+\sum_{v\in\mathcal{N}_i}f(R_{v\to i}^{t-1}).
\end{equation}
\\
For all $i$, set $T^t_{BP}(i)=+$ if $R_i^t\geq 0$, and set $T^t_{BP}(i)=-$ if $R_i^t< 0$
\caption{Local linearized belief propagation with vertex labels.}
\label{alg:BPlocal}
	\end{algorithm}

	\section{Properties of local, linearized belief propagation}
	We now consider the setting where $a,b$ and $c$ grow large. In this regime, we give specific performance guarantees on the success probability of the local belief propagation algorithm.
	Define
	\begin{equation*}
	R = \begin{bmatrix}
	pa &(1-p) b\\ pb &(1-p) c
	\end{bmatrix}.
	\end{equation*}
	We focus on the regime where $d\lambda_2^2$ is fixed,where $\lambda_2$ is the smallest eigenvalue of $R$ and define $\lambda=d\lambda_2^2$. We let the average degree $d\to\infty$. Then, $\lambda<1$ corresponds to the Kesten Stigum bound~\cite{caltagirone2017}. Furthermore, we assume that the average degree in each community is equal, so that~\eqref{eq:avgdeg} holds. Under this assumption, $\lambda_2=1-b$. Thus, if we let $b=1-\varepsilon$, then in the regime we are interested in, $d=\lambda/\varepsilon^2$. Then also
	\begin{equation}\label{eq:abc}
	a=1+\frac{1-p}{p}\varepsilon,\quad b=1-\varepsilon,\quad c=1+\frac{p}{1-p}\varepsilon.
	\end{equation}
	Define 
\begin{align}
\alpha_0&=0,\nonumber\\
\alpha_t&=G(\alpha_{t-1}),\quad t\geq 1,\label{eq:murec}
\end{align}
where
\begin{equation}\label{eq:Gmu}
G(\alpha ) = \frac{\lambda}{p^2} \Exp{\frac{1}{1-p+p\me^{U_-+\sqrt{\alpha}Z-\alpha/2}}-1}.
\end{equation}
Here $Z$ is a $\mathcal{N}(0,1)$ random variable, and $U_-$ is a random variable independent of $Z$ which takes values $\log(\mu(\ell)/\nu(\ell))$ with probability $\nu(\ell)$. 
Let 
\begin{equation}
Q(x)=\int_{x}^{\infty}\frac{1}{\sqrt{2\pi}}\me^{-y^2/2}\dd y,
\end{equation}
and let $U_+$ denote a random variable which takes values $\log(\mu(\ell)/\nu(\ell))$ with probability $\mu(\ell)$. Then the following theorem gives the success probability of Algorithm~\ref{alg:BPlocal} in terms of the function $Q$ and a fixed point of $G$ and compares it with the performance of the optimal estimator $T^{opt}$.

\begin{theorem}\label{thm:sbmlarge}
	Let $T_{BP}^t$ denote the estimator given by Algorithm~\ref{alg:BPlocal} up to depth $t$. Then, 
	\begin{equation}\label{eq:succBP}
	\begin{aligned}[b]
	& \liminf_{d\to\infty}\liminf_{n\to\infty} P_{succ}(T_{BP}^{t})\\
	& = \Exp{Q\left(\frac{U_+-\alpha_t/2}{\sqrt{\alpha_t}}\right)}+\Exp{Q\left(\frac{-U_--\alpha_t/2}{\sqrt{\alpha_t}}\right)}-1.
	\end{aligned}
	\end{equation}
Furthermore, if $G(\mu)$ has a unique fixed point, then 
\begin{equation}
\begin{aligned}[b]
& \liminf_{d\to\infty}\liminf_{n\to\infty} P_{succ}(T^{opt})\\
& =\Exp{Q\left(\frac{U_+-\alpha_\infty/2}{\sqrt{\alpha_\infty}}\right)}+\Exp{Q\left(\frac{-U_--\alpha_\infty/2}{\sqrt{\alpha_\infty}}\right)}-1,
\end{aligned}
\end{equation}
and the estimator of Algorithm~\ref{alg:BPlocal} is asymptotically optimal.
\end{theorem}
We now comment on the results and its implications. 

\paragraph{Special cases of $G(\alpha)$}
The function $G(\alpha)$ has been investigated for two special cases of the labeled stochastic block model. Analyzing $G(\alpha)$ for these special cases already turned out to be difficult, but some conjectures on its behavior have been made based on simulations.
In~\cite{mossel2016}, it was conjectured that for the special case where $p=1/2$ and the vertex labels are noisy versions of the spins, the function $G(\alpha)$ only has one fixed point for all possible values of $\lambda$. Thus, Algorithm~\ref{alg:BPlocal} is conjectured to perform optimally for the symmetric stochastic block model with noisy spins as vertex labels. 

The asymmetric stochastic block model where the information about the community memberships carried by the vertex labels goes to zero was studied in~\cite{caltagirone2017}. Instead of noisy spins as labels, a fraction of $\beta$ vertices reveals their true spins while the other vertices carry an uninformative label, where $\beta$ tends to zero as the graph grows large.
In that setting, it was conjectured that the function $G(\alpha)$ may have 2 or 3 fixed points for small values of $p$ and $\lambda<1$. 

\paragraph{Influence of the initial beliefs on the performance of Algorithm~\ref{alg:BPlocal}}
When $G(\alpha)$ has more than one fixed point, the success probability of Algorithm~\ref{alg:BPlocal} corresponds to the smallest fixed point of $G$. If in the belief propagation initialization the true unknown beliefs are used, the success probability of the algorithm corresponds to the largest fixed point of $G$. Since $G$ is increasing, this also implies that the success probability when initializing with the true unknown beliefs is higher than the success probability of Algorithm~\ref{alg:BPlocal}.

\paragraph{Multiple fixed points of $G$}
Figures~\ref{fig:Gfar} and~\ref{fig:Gclose} show that in the setting of Theorem~\ref{thm:sbmlarge} where information in the labels about the vertex spins does not vanish, the function $G(\alpha)$ may have more than one fixed point, even when the probability of observing the correct label does not go to $1/2$ as $n\to\infty$. This is very different from the special case where $p=1/2$, where the function $G(\alpha)$ was conjectured to have at most one fixed point~\cite{mossel2016}. Indeed, Figures~\ref{fig:Gfarsym} and~\ref{fig:Gclosesym} show that for the symmetric stochastic block model $G(\alpha)$ only contains one fixed point. 
For the asymmetric stochastic block model on the other hand, there is a region of parameters where Algorithm~\ref{alg:BPlocal} may not achieve the highest possible accuracy among all algorithms. Belief propagation initialized with the true beliefs corresponds to the highest fixed point of $G$, and thus results in a better estimator than belief propagation initialized with beliefs based on the vertex labels. In this case, exhaustive search of the partition space may still find all fixed points of the belief propagation algorithm, of which one corresponds to the fixed point having maximal overlap with the true partition. However, whether this fixed point can be distinguished from the other fixed points without knowledge of the true community spins is unknown. If this is possible, this would indicate a phase where community detection is possible, but computationally hard. If the fixed point is indistinguishable from the other fixed points, even exhaustive search of the partition space will not result in a better partition. In the asymmetric stochastic block model without vertex labels, it was shown that it is sometimes indeed possible to detect communities even underneath the Kesten-Stigum threshold~\cite{neeman2014} (in non-polynomial time). It would be interesting to see in which cases this also holds for the stochastic block model including vertex labels. 

\paragraph{Increasing vertex label information}
Interestingly, Figure~\ref{fig:Gclose} shows an example where the lowest and the highest fixed point of $G$ are stable, but the middle fixed point is unstable. Thus, to converge to a fixed point corresponding to a better correlation with the network partition than initializing at $\alpha=0$, the initial beliefs should correspond to an $\alpha$-value that is equal to or larger than the second fixed point of $G$ in this example.
Note that the case $\beta=0.5$ is similar to the community detection problem without extra vertex information, because for $\beta=0.5$ the vertex labels are independent of the community memberships. Thus, in the asymmetric stochastic block model, there is a fixed point of $G$ corresponding to a partition with non-trivial overlap with the true partition, but the BP algorithm does not find this partition when initialized with random beliefs. The same situation occurs when the information about the community membership carried by the vertex labels is small (for example when $\beta=0.48$). However, when the information carried by the vertex labels is sufficiently large, the BP algorithm starts to converge to the largest fixed point of $G$, and the BP algorithm performs optimally. Thus, including node covariates in the BP algorithm for the asymmetric stochastic block model may change the number of fixed points, and therefore significantly improve the performance of the BP algorithm.

\begin{figure}[htb]
	\centering
	\begin{subfigure}{0.48\linewidth}
		\centering
		\includegraphics[width=\textwidth]{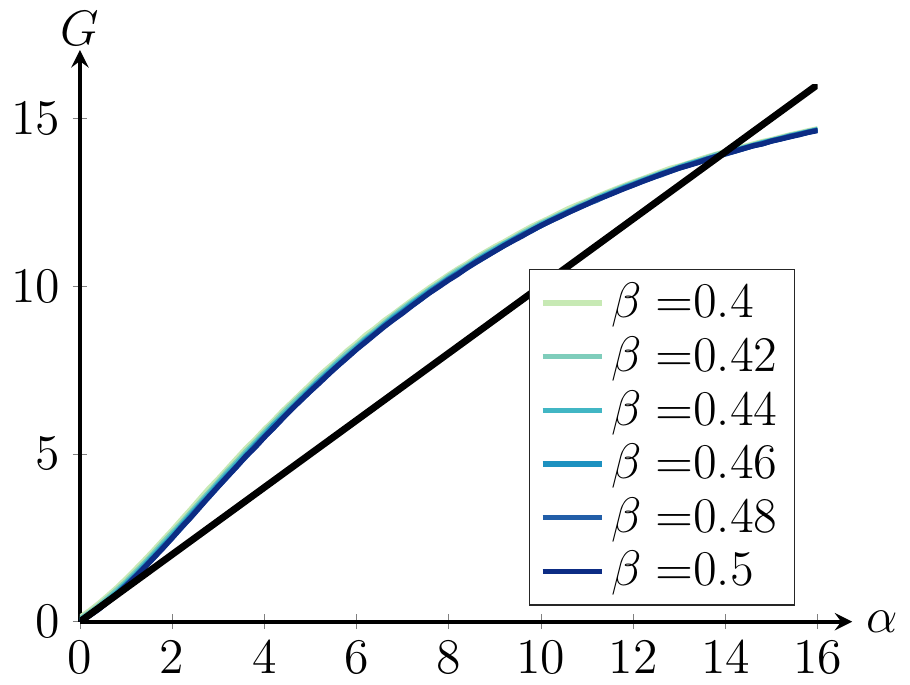}
		\caption{$p=0.05$}	
		\label{fig:Gfar}
	\end{subfigure}
	\hspace{0.1cm}
	\begin{subfigure}{0.48\linewidth}
		\centering
		\includegraphics[width=\textwidth]{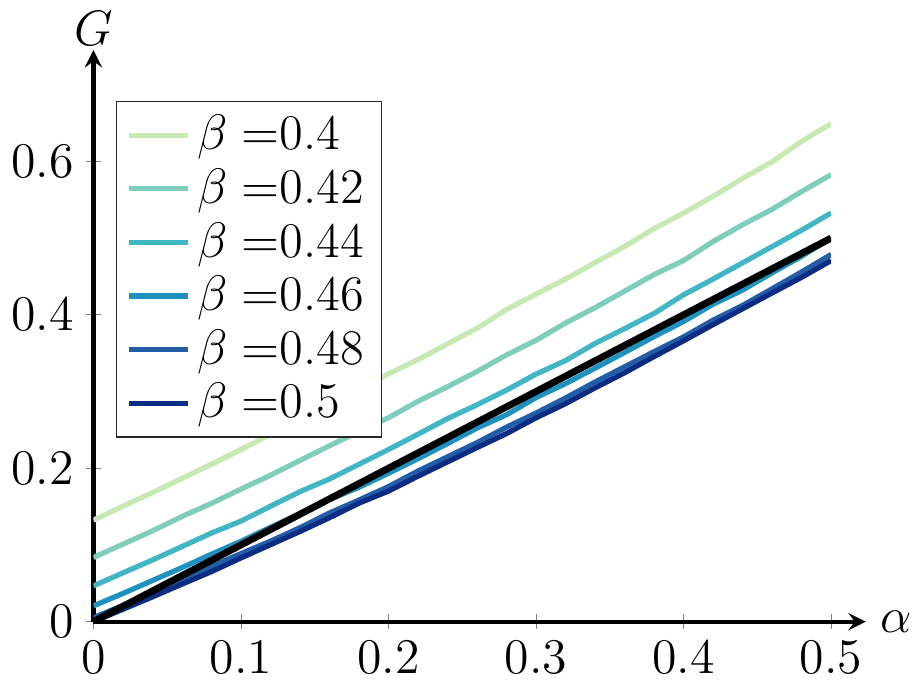}
		\caption{$p=0.05$, zoomed in}	
		\label{fig:Gclose}
	\end{subfigure}
	
	\begin{subfigure}{0.48\linewidth}
		\centering
		\includegraphics[width=\textwidth]{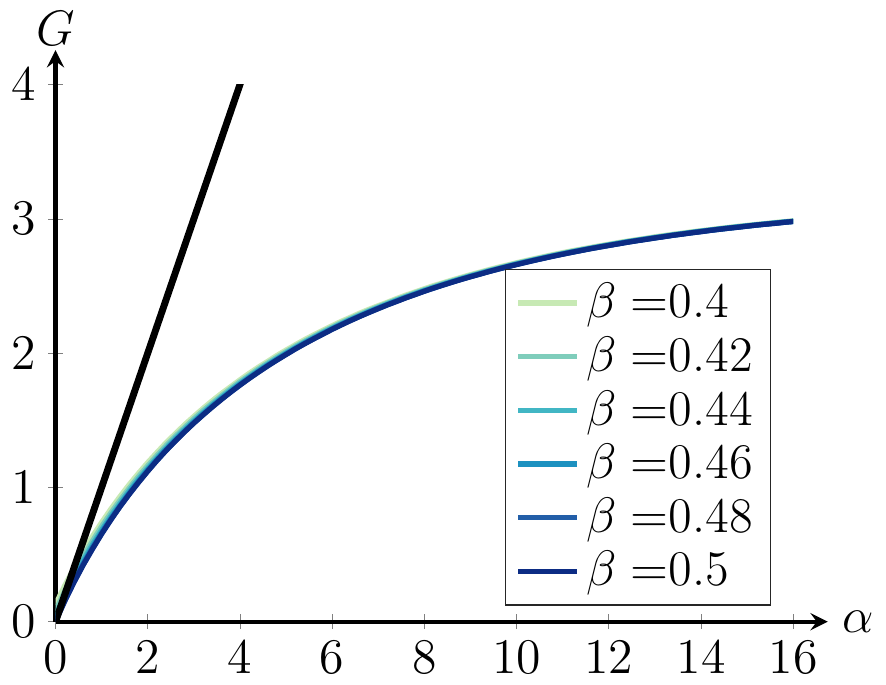}
		\caption{$p=0.5$}	
		\label{fig:Gfarsym}
	\end{subfigure}
	\hspace{0.1cm}
	\begin{subfigure}{0.48\linewidth}
		\centering
		\includegraphics[width=\textwidth]{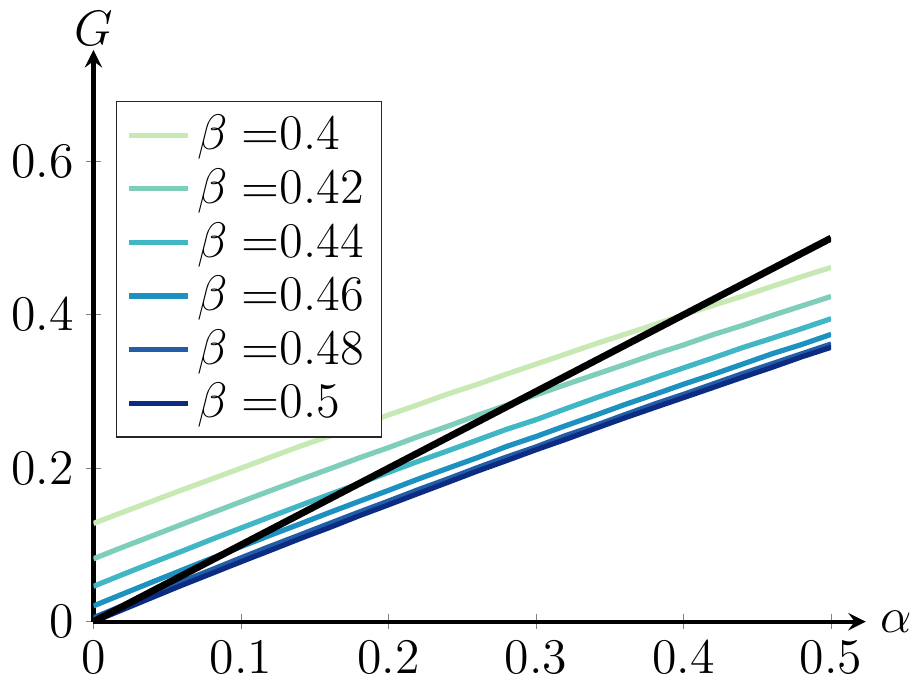}
		\caption{$p=0.5$, zoomed in}	
		\label{fig:Gclosesym}
	\end{subfigure}
	\caption{The function $G(\alpha)$ for $\lambda=0.8$ and noisy labels ($\mu(\ell_1)=\beta$, $\mu(\ell_3)=1-\beta$, $\nu(\ell_2)=\beta$, $\nu(\ell_3)=1-\beta$) for various values of $\beta$. The black line is the line $y=x$.}
	\label{fig:Gfix}
\end{figure}

\paragraph{Success probability}
Figures~\ref{fig:psuccpnoise} and~\ref{fig:psuccpreveal} plot the success probability of Algorithm~\ref{alg:BPlocal} given by equation~\eqref{eq:succBP} against $p$ for the case of noisy labels and revealed labels respectively. We see that for small and large values of $p$, these is a rapid increase in the success probability. This increase is caused by the shape of $G$, shown in Figures~\ref{fig:Gfar} and~\ref{fig:Gfarsym} for the setting with noisy labels. The location of the fixed point of $G$ is much closer to the origin for $p=0.5$ than for $p=0.05$. This difference causes the increase in the success probability.

Figures~\ref{fig:psucclnoise} and~\ref{fig:psucclreveal} show the success probability given by~\eqref{eq:succBP} as a function $\lambda$ for unbalanced communities. Here we also see that there is a small range of $\lambda<1$ where the success probability increases rapidly in $\lambda$. 
Figure~\ref{fig:BPlabels} shows that the accuracy obtained in Theorem~\ref{thm:sbmlarge} is higher than the accuracy that is obtained when only using the vertex labels to distinguish the communities, even underneath the Kesten-Stigum threshold $\lambda<1$.

\begin{figure}[htb]
	\centering
	\begin{subfigure}[t]{0.48\linewidth}
		\centering
		\includegraphics[width=\textwidth]{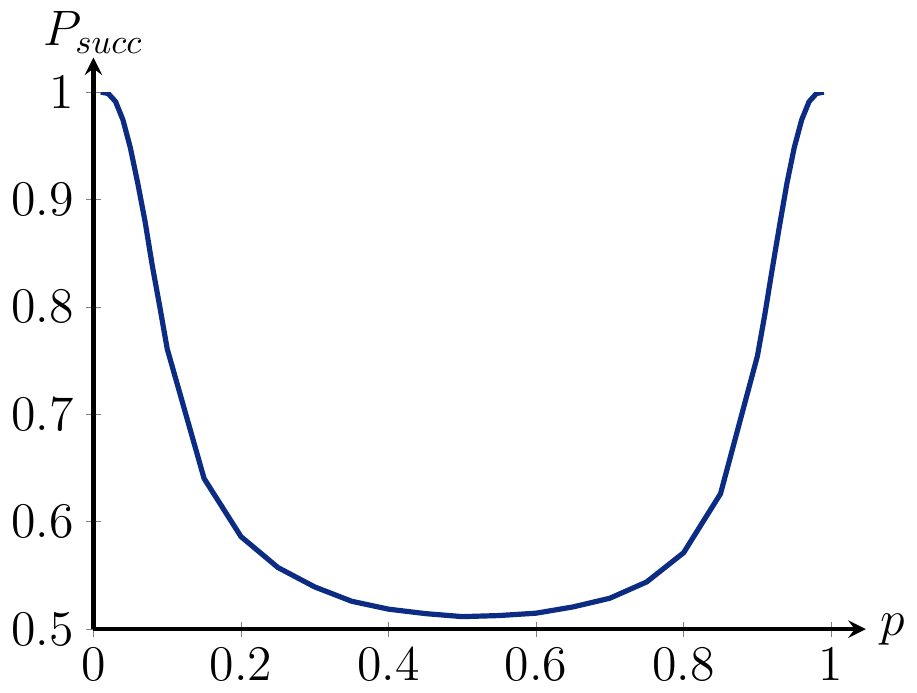}
		\caption{Noisy vertex labels: $\mu(\ell_1)=0.55=1-\mu(\ell_2)$, $\nu(\ell_2)=0.85=1-\nu(\ell_1)$.}
		\label{fig:psuccpnoise}
	\end{subfigure}
\hspace{0.1cm}
	\begin{subfigure}[t]{0.48\linewidth}
		\centering
		\includegraphics[width=\textwidth]{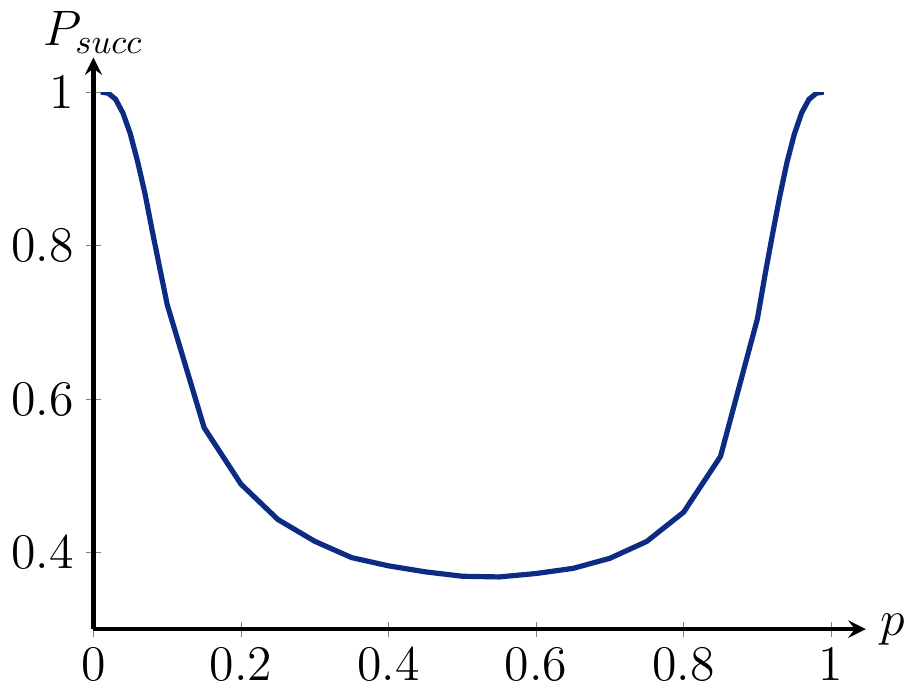}
		\caption{Revealed vertex labels: $\mu(\ell_1)=0.1=1-\mu(\ell_3)$, $\nu(\ell_2)=0.05=1-\nu(\ell_3)$.}
		\label{fig:psuccpreveal}
	\end{subfigure}
\caption{$P_{succ}$ as a function of $p$ for $\lambda=0.8$}
\end{figure}
\begin{figure}[htbp]
	\centering
	\begin{subfigure}[t]{0.48\linewidth}
		\centering
		\includegraphics[width=\textwidth]{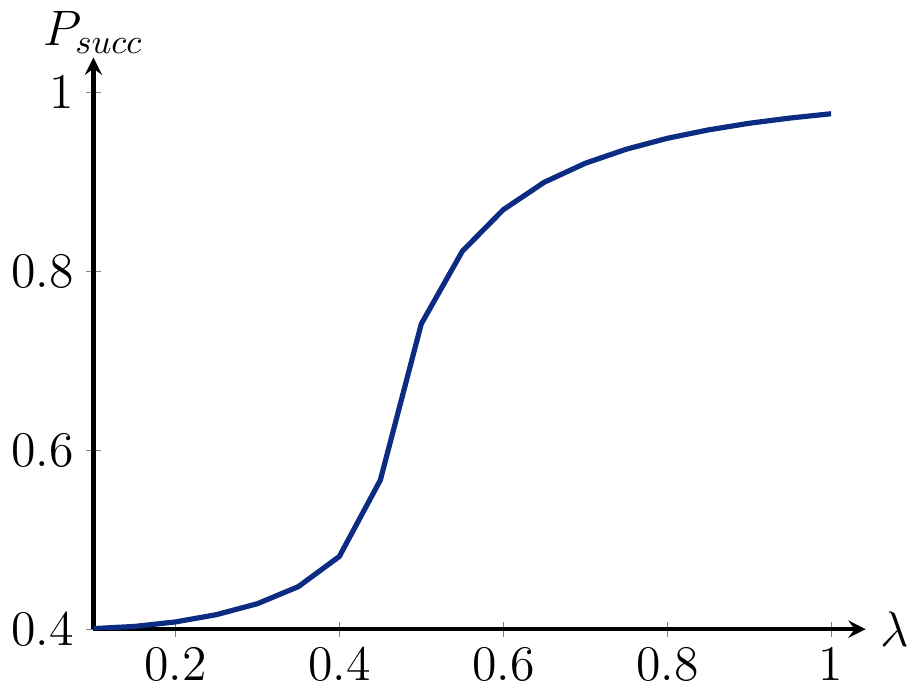}
		\caption{Noisy vertex labels: $\mu(\ell_1)=0.55=1-\mu(\ell_2)$, $\nu(\ell_2)=0.85=1-\nu(\ell_1)$.}
		\label{fig:psucclnoise}
	\end{subfigure}
	\hspace{0.1cm}
	\begin{subfigure}[t]{0.48\linewidth}
		\centering
		\includegraphics[width=\textwidth]{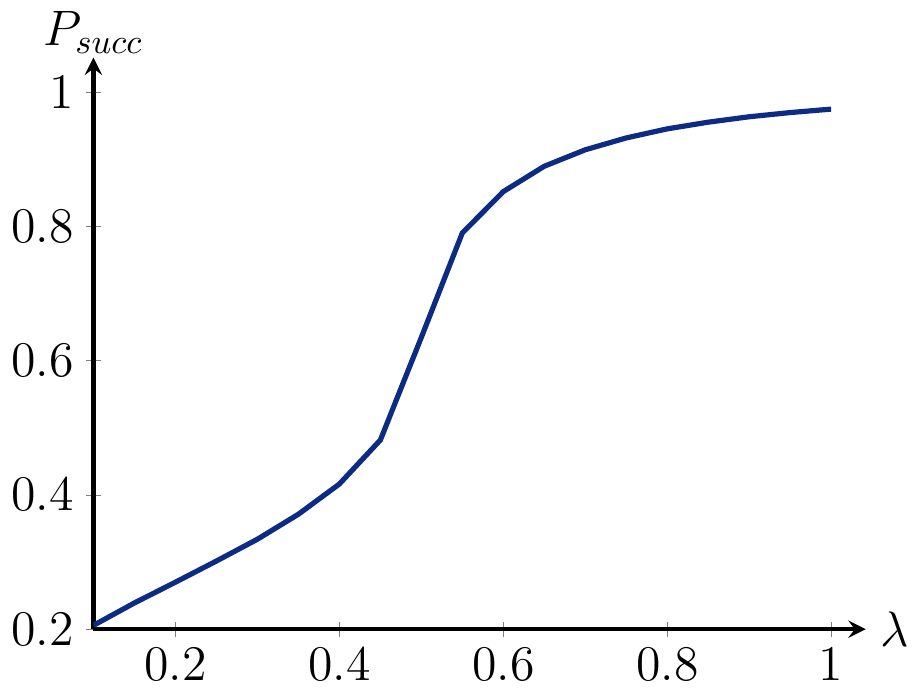}
		\caption{Revealed vertex labels: $\mu(\ell_1)=0.1=1-\mu(\ell_3)$, $\nu(\ell_2)=0.05=1-\nu(\ell_3)$.}
		\label{fig:psucclreveal}
	\end{subfigure}
\caption{$P_{succ}$ as a function of $\lambda$ for $p=0.05$}
\end{figure}

\begin{figure}
	\centering
	\includegraphics[width=0.45\textwidth]{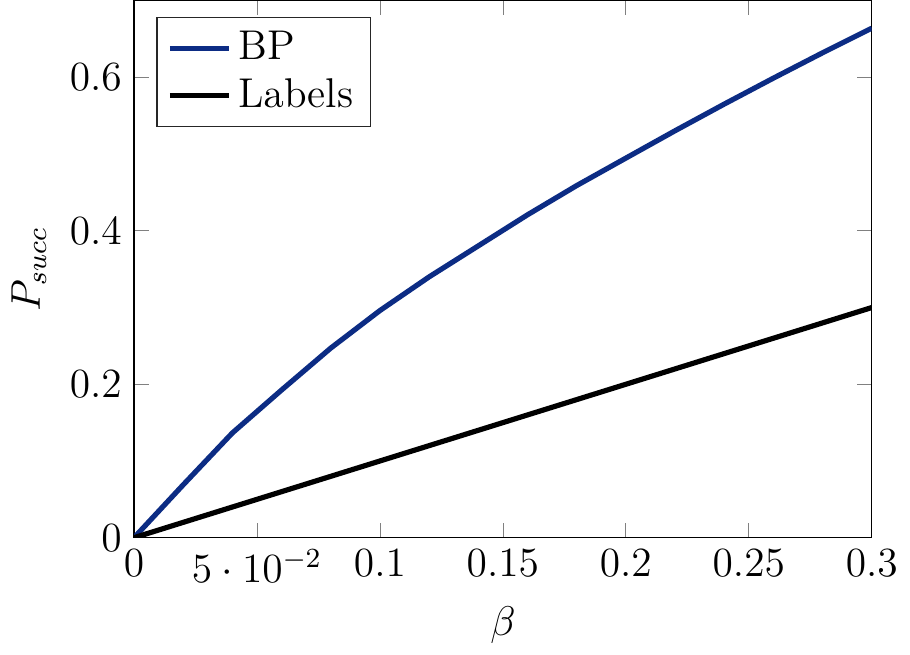}
	\caption{Success probability of Algorithm~\ref{alg:BPlocal} and the success probability when only using the vertex labels for $p=0.5$, $\lambda=0.8$,  $\mu(\ell_1)=0.5+\beta=1-\mu(\ell_2)$ and $\nu(\ell_1)=0.5-\beta=1-\nu(\ell_2)$.}
	\label{fig:BPlabels}
\end{figure}

\section{Proof of Theorem~\ref{thm:sbmlarge}}
Because the SBM is locally tree-like, we first investigate Algorithm~\ref{alg:BPlocal} on a Galton-Watson tree defined in Section~\ref{sec:GW}, where we study the recursion~\eqref{eq:xirecurs}. 
Denote by $\xi_r^\pp$ the value of $\xi_r$ for a randomly chosen vertex in community + and define $\xi_r^\mm$ similarly. Then, $\xi_0^\pp\overset{d}{=}U_++w$ and $\xi_0^\mm\overset{d}{=}U_-+w$. We first investigate the distribution of $\xi_1$. 

\begin{lemma}\label{lem:xi1}
	As $d\to\infty$,
	\begin{align}
		\lim_{d\to\infty} \xi_1^\pp& \wlim U_++w+\mathcal{N}(\alpha_1/2,\alpha_1),\\
			\lim_{d\to\infty} \xi_1^\mm& \wlim U_++w-\mathcal{N}(\alpha_1/2,\alpha_1),
	\end{align}
	where $\wlim$ denotes convergence in the Wasserstein metric (see for example~\cite[Section 2]{gibbs2002}).
\end{lemma}
\begin{proof}
From the recursion~\eqref{eq:murec} we obtain
\begin{equation*}
\begin{aligned}[b]
\alpha_1& = G(0)=\frac{\lambda}{p^2}\sum_{\ell}\nu(\ell)\frac{p-p\frac{\mu(\ell)}{\nu(\ell)}}{1-p+p\frac{\mu(\ell)}{\nu(\ell)}}\\
& =\frac{\lambda}{p}\sum_{\ell}\nu(\ell)\frac{\nu(\ell)-\mu(\ell)}{(1-p)\nu(\ell)+p\mu(\ell)}\\
& = \frac{\lambda}{p}\sum_{\ell}\frac{\nu(\ell)^2-\mu(\ell)\nu(\ell)}{(1-p)\nu(\ell)+p\mu(\ell)}+\mu(\ell)-\nu(\ell)\\
& = \frac{\lambda}{p}\sum_{\ell}\frac{p\nu(\ell)^2+p\mu(\ell)^2-2p\nu(\ell)\mu(\ell)}{(1-p)\nu(\ell)+p\mu(\ell)}\\
& =\lambda\sum_{\ell}\frac{(\mu(\ell)-\nu(\ell))^2}{p\mu(\ell)+(1-p)\nu(\ell)}.
\end{aligned}
\end{equation*}
Furthermore, we can write $\xi_1^\pp$ as 
\begin{equation}\label{eq:xiplus}
\xi_1^\pp\overset{(d)}{=}U_++w+\sum_{\ell'}N_{\ell'}f(h(\ell')+w),
\end{equation}
where $N_{\ell'}\sim\text{Poisson}(d(pa\mu(\ell')+(1-p)b\nu(\ell')))$, independent of $U_+$.
Subtracting and adding the mean of the Poisson variable yields
\begin{equation}\label{eq:xi1}
\begin{aligned}[b]
	\xi_1^\pp& \overset{(d)}{=}\sum_{\ell'}(N_{\ell'}-d(pa\mu(\ell')+(1-p)b\nu(\ell')))f(h(\ell')+w)\\
	& \quad +\sum_{\ell'}d(pa\mu(\ell')+(1-p)b\nu(\ell')))f(h(\ell')+w)+U_+.
\end{aligned}
\end{equation}
In the regime we are interested in, $a=1+\frac{1-p}{p}\varepsilon$, $b=1-\varepsilon$, $c=1+\frac{p}{1-p}\varepsilon$ and $d=\lambda/\varepsilon^2$ for some $\varepsilon>0$ (see~\eqref{eq:abc}). Then, Taylor expanding $f(h(\ell)+w)$ around $\varepsilon=0$ results in
\begin{equation*}
\begin{aligned}[b]
f(h(\ell)+w)& =\log\left(\frac{(1+\varepsilon\frac{1-p}{p})p\mu(\ell)+(1-\varepsilon)(1-p)\nu(\ell)}{(1-\varepsilon)p\mu(\ell)+(1+\varepsilon\frac{p}{1-p})(1-p)\nu(\ell)}\right)\\
& =\varepsilon \frac{\mu(\ell)-\nu(\ell)}{p\mu(\ell)+(1-p)\nu(\ell)}\\
& \quad +\frac{\varepsilon^2}{2} \frac{(2p-1)(\mu(\ell)-\nu(\ell))^2}{(p\mu(\ell)+(1-p)\nu(\ell))^2}+O(\varepsilon^3).
\end{aligned}
\end{equation*}
This shows that the last term in~\eqref{eq:xi1} can be rewritten as
\begin{equation*}
\begin{aligned}[b]
	& \sum_{\ell'}d(pa\mu(\ell')+b(1-p)\nu(\ell')))f(h(\ell') \\
	&= \frac{\lambda}{\varepsilon^2} \sum_{\ell'}((1+\tfrac{1-p}{p}\varepsilon)p\mu(\ell')\\
	& \quad +(1-\varepsilon)(1-p)\nu(\ell')))f(h(\ell)+w)\\
	& =\frac{\lambda}{2}  \sum_{\ell'}\frac{(\mu(\ell')-\nu(\ell'))^2}{p\mu(\ell')+(1-p)\nu(\ell')}+O(\varepsilon) \\
	& = \alpha_1/2+O(\varepsilon).
\end{aligned}
\end{equation*}
By~\cite[Corollary A3]{caltagirone2017},
\begin{equation*}
\frac{N_{\ell'}-d(ap\mu(\ell')+b(1-p)\nu(\ell'))}{\sqrt{d(ap\mu(\ell')+b(1-p)\nu(\ell'))}}\wlim \mathcal{N}(0,1),
\end{equation*}
as $d\to\infty$. Then, 
\begin{equation*}
\sum_{\ell'}N_{\ell'}-d(ap\mu(\ell')+b(1-p)\nu(\ell')))f(h(\ell'))\wlim\mathcal{N}(0,\alpha_1).
\end{equation*}
Thus, as $d\to\infty$
\begin{equation*}
	\xi_1^\pp \wlim U_++w+ \mathcal{N}(\alpha_1/2,\alpha_1)
\end{equation*}
and similar arguments prove the lemma for $\xi_1^\mm$.
\end{proof}

We now proceed to the distribution of $\xi_r$ for $r>1$ by using induction. 
\begin{lemma} \label{lem:induction}
	Assume that 
\begin{align}
		\xi_{r}^\pp & \wlim U_++w+ \mathcal{N}(\alpha_{r}/2,\alpha_r)\\
		\xi_{r}^\mm & \wlim U_- +w- \mathcal{N}(\alpha_{r}/2,\alpha_r)
\end{align}
for some  $r\geq 1$. Then,
\begin{align}
\xi_{r+1}^\pp &\wlim U_++w+ \mathcal{N}(\alpha_{r+1}/2,\alpha_{r+1})\\
\xi_{r+1}^\mm&\wlim  U_-+w- \mathcal{N}(\alpha_{r+1}/2,\alpha_{r+1})
\end{align}
as $d\to\infty$.
\end{lemma}
\begin{proof}
Define
\begin{equation}
\gamma_r^{(s_0)}= \xi_r^{(s_0)}-h(L_{s_0})-w,
\end{equation}
and define $\gamma_r^\pp$ as the value of $\gamma_r^{(s_0)}$ for a randomly chosen $s_0$ in community +.
We start by investigating the first moment of $\gamma_{r+1}^\pp$. Using Walds equations, we obtain
\begin{equation}\label{eq:expxir}
\begin{aligned}[b]
\Exp{\gamma_{r+1}^\pp}&= dap\Exp{f(\xi_{r}^\pp)}+d b(1-p)\Exp{ f(\xi_{r}^\mm)}\\
& =\tfrac{\lambda}{\varepsilon^2}(p+\varepsilon(1-p))\Exp{f(\xi_{r}^\pp)}\\
& \quad +\tfrac{\lambda}{\varepsilon^2} (1-\varepsilon)(1-p)\Exp{f(\xi_{r}^\mm)},
\end{aligned}
\end{equation}
where the second line uses~\eqref{eq:abc}. 
We then use that~\cite[Eq. (A4)]{caltagirone2017} 
\begin{equation*}
\begin{aligned}[b]
f(x)&=\log\Big(1+\varepsilon\frac{\me^x}{p(1+\me^x)}+\varepsilon^2\frac{\me^x}{p(1+\me^x)}+O(\varepsilon^3)\Big)\\
& \quad -\log\Big(1+\varepsilon\frac{1}{(1-p)(1+\me^x)}\\
& \quad +\varepsilon^2\frac{1}{(1-p)(1+\me^x)}+O(\varepsilon^3)\Big).
\end{aligned}
\end{equation*}
Taylor expanding $\log(1+x)$ then results in 
\begin{equation}\label{eq:fappr}
\begin{aligned}[b]
f(x)& = \varepsilon\frac{\me^{x}(1+\varepsilon)}{p(1+\me^x)}-\varepsilon\frac{1+\varepsilon}{(1-p)(1+\me^x)}-\varepsilon^2\frac{1-\me^{2x}}{2p^2(1+\me^x)^2}\\
& \quad +\varepsilon^2\frac{1}{2(1-p)^2(1+\me^x)^2}+O(\varepsilon^3).
\end{aligned}
\end{equation}
For all bounded continuous functions $g(x)$ by~\cite[Lemma A6]{caltagirone2017}
\begin{equation}
(1-p)\Exp{g(\xi_r^\mm)}=p\Exp{g(\xi_r^\pp)\me^{-\xi_r^\pp}}.
\end{equation}
Denote $h_1(x)=(1+\me^x)^{-1}$ and $h_2(x)=\me^{x}(1+\me^x)^{-1}$. Then,
\begin{align*}
	(1-p)\Exp{h_2(\xi_r^{\scriptscriptstyle{(-)}})}+p\Exp{h_2(\xi_r^\pp)} & = p,\\
	(1-p)\Exp{h_1(\xi_r^{\scriptscriptstyle{(-)}})}+p\Exp{h_1(\xi_r^\pp)} & =1- p,\\
	(1-p)\Exp{h_2(\xi_r^\mm)^2}+p\Exp{h_2(\xi_r^\pp)^2} & = p\Exp{h_2(\xi_r^\pp)},\\
		(1-p)\Exp{h_1(\xi_r^\mm)^2}+p\Exp{h_1(\xi_r^\pp)^2} & = (1-p)\Exp{h_2(\xi_r^\mm)},
\end{align*}
Combining this with~\eqref{eq:expxir} and~\eqref{eq:fappr} gives
\begin{equation*}
\begin{aligned}[b]
&\Exp{\gamma_{r+1}^\pp}
= 
	\tfrac{\lambda}{\varepsilon^2}\Big(\varepsilon(1-1)+\varepsilon^2\Big(\frac{1}{2p}\Exp{h_2(\xi_r^\pp)}\\
	& \quad +\frac{1}{2(1-p)}\Exp{h_1(\xi_{r}^\mm)}-\frac{1}{p}\Exp{h_2(\xi_r^\mm)}\Big)\Big)+O(\varepsilon)\\
	&=\lambda\Big( \frac{1}{2p}-\frac{1-p}{2p^2}\Exp{h_2(\xi_r^\mm)}\\
	& \quad +\frac{1}{2(1-p)}\Exp{h_1(\xi_r^\mm)}-\frac{1}{p}\Exp{h_2(\xi_r^\mm)}\Big)+O(\varepsilon)\\
	&=\lambda\Big(\frac{1}{2p}-\frac{1+p}{2p^2}+\frac{1+p}{2p^2}\Exp{h_1(\xi_r^\mm)}\\
	& \quad +\frac{1}{2(1-p)}\Exp{h_1(\xi_r^\mm)}\Big)+O(\varepsilon)\\
	& = \frac{\lambda}{2p^2}\Exp{\frac{1}{(1+\me^{\xi_r^\mm})(1-p)}-1} +O(\varepsilon).
\end{aligned}
\end{equation*}
Combining this with the induction hypothesis results in
\begin{equation*}
\begin{aligned}[b]
&\Exp{\gamma_{r+1}^\pp}	\\
& = \frac{\lambda}{2p^2}\Exp{\frac{1}{(1-p)(1+\me^{U_-+w+\sqrt{\alpha_r}Z-\alpha_r/2})}-1} +O(\varepsilon)\\
& =\frac{\lambda}{2p^2}\Exp{\frac{1}{1-p+p\me^{U_-+\sqrt{\alpha_r}Z-\alpha_r/2}}-1} +O(\varepsilon)\\
& =\tfrac{1}{2}G(\alpha_r)+O(\varepsilon).
\end{aligned}
\end{equation*}
For the variance, we obtain using Walds equation
\begin{equation}
\begin{aligned}[b]
& \Var{{\gamma_{r+1}^\pp}} = dap\Exp{f(\xi_{r}^\pp)^2}+d b(1-p)\Exp{f(\xi_{r}^\mm)^2}\\
 & = 2\lambda (1+\varepsilon)\Big(\frac{1}{p^2}\Exp{h_2(\xi_r^\pp)^2}+\frac{1}{(1-p)^2}\Exp{h_1(\xi_r^\pp)^2} \\
 &\quad  +\frac{2}{p(1-p)}\Exp{h_1(\xi_r^\pp)^2\me^{\xi_r^\pp}}\Big)\\
 &\quad +(1-\varepsilon)\Big(\frac{1}{p^2}\Exp{h_2(\xi_r^\mm)^2}+\frac{1}{(1-p)^2}\Exp{h_1(\xi_r^\mm)^2}\\
 &\quad  +\frac{2}{p(1-p)}\Exp{h_1(\xi_r^\mm)^2\me^{\xi_r^\mm}}\Big)+O(\varepsilon),
\end{aligned}
\end{equation}
where we used~\eqref{eq:fappr} again. Similar computations as for the expected value then lead to
\begin{equation}
\Var{{\gamma_{r+1}^\pp}}=  G(\alpha_r)=\alpha_{r+1}.
\end{equation}
Thus, the first and second moment of $\gamma_r^\pp$ are of the correct size. 
The proof that $\gamma_{r+1}^\pp$ converges to a normal distribution then follows the exact same lines as the proof in~\cite[Proposition 23]{caltagirone2017}. 
\end{proof}

We now study the total variation distance of a labeled Galton-Watson tree where the root is in community $+$ and a Galton-Watson tree where the root is in community $-$.
\begin{lemma}\label{lem:dtv}
	Let $P_+^{(t)}$ and $P_-^{(t)}$ denote the conditional distributions of $\mathcal{T}_t^{(s_0,L)}$ conditionally on the spin of $s_0$ being + and - respectively. Then,
	\begin{equation}
		\begin{aligned}[b]
		\lim_{d\to\infty}d_{TV}(P_+^{(t)},P_-^{(t)})& = \Exp{Q\left(\frac{-U_+-\alpha_t/2}{\sqrt{\alpha_t}}\right)}\\
		& + \Exp{Q\left(\frac{U_--\alpha_t/2}{\sqrt{\alpha_t}}\right)}-1
		\end{aligned}
	\end{equation}
\end{lemma}
\begin{proof}
By~\eqref{eq:dtvpsucc}, the term on the left hand side is the same as the success probability of the estimator of Algorithm~\ref{alg:BPlocal} on a Galton-Watson tree. 
Using that $\xi_t^\pp$ and $\xi_t^\mm$ converge to normal distributions in the large graph limit, we then obtain for the total variation distance that
\begin{equation*}
\begin{aligned}[b]
&d_{TV}(P_+^{(t)},P_-^{(t)})\\
& =P_{succ}^{(GW)}(T_{BP}^t)=\Prob{\xi_t^\pp\geq 0}+\Prob{\xi_t^\mm\leq 0}-1\\
& = \Exp{Q\left(\frac{-U_+-\alpha_t/2}{\sqrt{\alpha_t}}\right)}+ \Exp{Q\left(\frac{U_--\alpha_t/2}{\sqrt{\alpha_t}}\right)}-1.
\end{aligned}
\end{equation*}
\end{proof}

Finally, we need to relate our results on the labeled Galton-Watson trees to the SBM. Denote by $G_t^{(s_0,L)}$ the subgraph of $G$ induced by all vertices at distance at most $t$ from vertex $s_0$. Let $\sigma_{G_t}$ denote the spins of all vertices in $G_t^{(s_0,L)}$. Similarly, let $\sigma_{\mathcal{T}_t}$ denote the spins of the vertices in $\mathcal{T}_t^{(s_0,L)}$. Then, the following Lemma can be proven analogously to~\cite{mossel2014}.
\begin{lemma}\label{lem:coupling}
	For $t=t(n)$ such that $a^t=n^{o(1)}$, there exists a coupling between $(G_t^{(s_0,L)},\sigma_{G_t})$ and $(\mathcal{T}_t^{(s_0,L)},\sigma_{{\mathcal{T}}_t})$ such that $(G_t^{(s_0,L)},\sigma_{G_t})=(\mathcal{T}_t^{(s_0,L)},\sigma_{{\mathcal{T}}_t})$ with high probability.
\end{lemma}

This lemma allows us to finish the proof of Theorem~\ref{thm:sbmlarge}.
\begin{proof}[Proof of Theorem~\ref{thm:sbmlarge}]
	On the event that $(G_t^{(s_0,L)},\sigma_{G_t})=(\mathcal{T}_t^{(s_0,L)},\sigma_{{\mathcal{T}}_t})$, the estimator of Algorithm~\ref{alg:BPlocal} is the same as the estimator based on the sign of $\xi_t$. Therefore, 
	\begin{equation*}
		\lim_{n\to\infty}P_{succ}(T^t_{BP})=P_{succ}^{(GW)}(T_{BP}^t),
	\end{equation*}
	so that 
	\begin{equation}
		\begin{aligned}[b]
		\lim_{d\to\infty}\lim_{n\to\infty}P_{succ}(T^t_{BP})& =\Exp{Q\left(\frac{-U_+-\alpha_t/2}{\sqrt{\alpha_t}}\right)}\\
		& + \Exp{Q\left(\frac{U_--\alpha_t/2}{\sqrt{\alpha_t}}\right)}-1,
		\end{aligned}
	\end{equation}
	which proves~\eqref{eq:succBP}. 
	
	To prove the second claim, we define estimator $\tilde{T}^t_{BP}$ on a tree of depth $t$ that does not only have access to the observed labels of the tree, but also to the vertex spins at depth $t$. Then, similar to~\cite[Lemma 3.9]{mossel2016}, we can show that this estimator performs at least as well as the optimal estimator on the Galton-Watson tree without revealed spins as $n\to\infty$. The analysis of this estimator follows the exact same lines as the analysis of Algorithm~\ref{alg:BPlocal}, except for the initial beliefs. Let $\zeta_r^\pp$ and $\zeta_r^\mm$ be defined similarly as $\xi_r^\pp$and $\xi_r^\mm$, but now given the true spins at depth $t$. 
	Define 
	\begin{align}
		\tilde{\alpha}_1&=\frac{\lambda}{p(1-p)},\\
		\tilde{\alpha}_r&=G(\tilde{\alpha}_{r-1}).
	\end{align}
	We now show that when $d\to\infty$
	\begin{align}\label{eq:zeta1p}
\zeta_1^\pp & \wlim U_++w+\mathcal{N}(\tilde{\alpha}_1/2,\tilde{\alpha}_1)\\
\zeta_1^\mm & \wlim U_-+w+\mathcal{N}(\tilde{\alpha}_1/2,\tilde{\alpha}_1).\label{eq:zeta1m}
	\end{align}
	Similar to~\eqref{eq:xiplus}, we can write
	\begin{equation}
		\zeta_1^\pp\overset{d}{=}U_++w+N_1\log(a/b)+N_2\log(b/c),
	\end{equation}
	with $N_1$ and $N_2$ Poisson random variables with parameters $dpa$ and $b(1-p)b$ respectively. Using~\eqref{eq:abc}, we obtain that $\log(a/b)=\tfrac{\varepsilon}{p}+\varepsilon^2\tfrac{2p-1}{2p^2}+O(\varepsilon^3)$ and $\log(c/d)=-\tfrac{\varepsilon}{1-p}+\varepsilon^2\tfrac{2p-1}{2(1-p)^2}+O(\varepsilon^3)$. Therefore,
	\begin{equation*}
	\begin{aligned}[b]
	& dpa\log(a/b)+d(1-p)b\log(c/d)\\
	& =\lambda \left(\frac{1}{p}+\frac{2p-1}{2p(1-p)}\right)+o(\varepsilon)\\
	& =\frac{\lambda}{2p(1-p)}+o(\varepsilon)=\tilde{\alpha}_1+o(\varepsilon).
	\end{aligned}
	\end{equation*}
	We can then use the same arguments as in Lemma~\ref{lem:xi1} to prove~\eqref{eq:zeta1p} and~\eqref{eq:zeta1m}. From there on, we can use Lemma~\ref{lem:induction}, with the value of $\alpha_1$ replaced by $\tilde{\alpha}_1$. Following the same lines of the proof of the analysis of $\xi$ then leads to
	\begin{equation*}
	\begin{aligned}[b]
	\lim_{d\to\infty}\lim_{n\to\infty}P_{succ}(\tilde{T}^t_{BP})& =\Exp{Q\left(\frac{-U_+-\tilde{\alpha}_t/2}{\sqrt{\tilde{\alpha}_t}}\right)}\\
	& + \Exp{Q\left(\frac{U_--\tilde{\alpha}_t/2}{\sqrt{\tilde{\alpha}_t}}\right)}-1.
	\end{aligned}
	\end{equation*}
	
	Similarly to~\cite[Lemma 7.4]{mossel2016}, we can show that $G(\alpha)$ is increasing and continuous. Therefore, if the function $G$ only has one fixed point, estimators $\tilde{T}^t_{BP}$ and $T^t_{BP}$ will provide the same accuracy as $t\to\infty$. 
\end{proof}

\section{Learning the model parameters}
Note that Algorithm~\ref{alg:BPlocal} uses knowledge of the parameters of the stochastic block model, $a,b$ and $c$, as well as the entire label distribution depending on the community spin, $\mu$ and $\nu$. In practice however, the parameters of the model that underlies some observed network are often unknown. When the parameters $a$, $b$ and $c$ of the stochastic block model are sufficiently large (larger than $\log(n)$), the communities can be recovered above the Kesten-Stigum threshold with a vanishing error fraction without knowledge of the model parameters by using a spectral method~\cite{lei2015,xu2014}. Let $\hat{\sigma}_i$ denote the estimated community spins. We can then estimate $\mu(\ell)$ by
\begin{equation}
\hat{\mu}(\ell)=\frac{\sum_{i:\hat{\sigma}_i=+}\ind{L_i=\ell}}{\sum_{i:\in [n]}\ind{\hat{\sigma}_i=+}}=\frac{\mu(\ell)np+o(1)}{np(1+o(1))}=\mu(\ell)+o(1),
\end{equation}
and $\nu(\ell)$ can be estimated with vanishing error as well. 

The spectral method described above only depends on the adjacency matrix. It is also possible to apply these spectral methods to a different matrix, that includes the adjacency matrix as well as the vertex labels. One example of such a matrix is the matrix $A+K$, where $A$ is the adjacency matrix, and $K$ is a kernel matrix based on the vertex labels. When every vertex has a number of vertex labels that grows in the network size $n$, a spectral method on the adjacency matrix with additive kernel is able to correctly identify a larger fraction of the spins correctly than a spectral algorithm based on the adjacency matrix only~\cite{binkiewicz2017,weng2016}. 

Another option is to study a multiplicative kernel, that is, a matrix of the form $A\circ K$, where $K$ is again some kernel matrix based on the vertex labels, and $\circ$ denotes element-wise multiplication. For example, we can take $K(i,j)=\ind{\ell_i=\ell_j}$. Then, $A\circ K$ is the adjacency matrix of the graph where all edges between vertices of different labels are removed. The remaining graph graph consists of several components, where the vertex label within each component is equal.  On average, the component corresponding to label $\ell$ contains $np\mu(\ell)$ vertices with spin $+$ and $n(1-p)\nu(\ell)$ of spin $-$. The average degree of the vertices with spin $+$ in this graph is therefore equal to $dp\mu(\ell)a+d(1-p)\nu(\ell)b$, whereas the average degree of the vertices with spin $-$ is equal to $dp\mu(\ell)b+d(1-p)\nu(\ell)c$. Using that $c=\tfrac{p}{1-p}(a-b)+b$ results in that the average degrees of vertices label $\ell$ with spin $+$ and $-$ are unequal if
\begin{equation}
p\mu(\ell)a\neq p\mu(\ell)b+p\nu(\ell)(a-b),
\end{equation}
so that they are unequal if $a\neq b$ and $\mu(\ell)\neq \nu(\ell)$. 
If this condition holds, then it is possible to infer the community spins of this connected component better than a random guess~\cite[Lemma 4]{caltagirone2017}. Then, in a regime where the average degree is at least logarithmic, we can get correlated reconstruction from the spectral technique of~\cite{xu2014} applied to the matrix with kernel multiplication without knowledge of the model parameters.
This method may even work underneath the Kesten-Stigum threshold, as in Example~\ref{ex:SBM}.
However, if the original SBM contained $K$ communities, this method finds $\abs{\mathcal{L}}K$ communities, one set of communities for each label. Thus, for each community $\sigma$ in the model, this method finds subcommunities $(\sigma,\ell)$ for all labels $\ell$, containing the vertices of community $\sigma$ with label $\ell$. 
One remaining question then is how to identify the different subcommunities that belong to the same original community. One possibility could be to identify the different parts of the planted communities based on estimates of the connection probabilities within the subcommunities and between the subcommunities. 

The kernel matrix $K(i,j)=\ind{\ell_i=\ell_j}$ is quite restrictive, but the above example shows that using multiplicative kernel matrices for community detection with vertex labels instead of additive kernel matrices seems promising. It would be interesting to investigate if other kernel matrices would result in a better performance. For example, the kernel matrix $K(i,j)=w_{\ell_i,\ell_j}$ for some weight matrix $w$ may result in better performance.

%
%

\section{Conclusion}
In this paper, we investigated a variant of the belief propagation (BP) algorithm for asymmetric stochastic block models with vertex labels. We find the probability that the belief propagation algorithm correctly classifies a vertex when the average degree of a vertex grows large. We show that in the asymmetric stochastic block model, the belief propagation algorithm initialized with beliefs based on the vertex labels may not always perform optimally. Belief propagation initialized with the true community memberships then results in a better partition. Whether it is possible to know that this partition is better than the partition obtained by initializing BP based on the vertex labels without knowing the true community partition, is an interesting direction for future research.

To determine the optimality or sub-optimality of BP in such situations, one possible approach could be to characterize directly the optimal accuracy that can be obtained by any feasible estimation procedure irrespective of its computational cost. Recent works have performed such a characterization in scenarios distinct from ours (eg~\cite{miolane2017,alaoui2018}, and references therein). It remains to be seen whether the approaches in these papers could be adapted to our present scenario.

Furthermore, the belief propagation algorithm uses knowledge of all parameters of the stochastic block model and the vertex label distribution. In general, such parameters are not known. Another fruitful direction for future research is therefore to investigate algorithms that do not need the model parameters as input, or algorithms that estimate the parameters given a network observation. For example, spectral methods including vertex covariates~\cite{binkiewicz2017}, methods based on maximum likelihood estimation~\cite{weng2016} or using belief propagation to find the parameters of the algorithm~\cite{decelle2011} may be interesting to investigate. 

Finally, it would be interesting to investigate the performance of a similar algorithm when more than two communities are present.

\bibliographystyle{abbrv}
\bibliography{referencesparis}

\end{document}